\documentclass[11pt]{article}
\usepackage{graphicx,amsmath,amsfonts,amssymb,bm,hyperref,url,breakurl,epsfig,epsf,color,fullpage,MnSymbol,mathbbol,fmtcount,algorithmic,algorithm,semtrans,cite,caption,subcaption,
  multirow}

\usepackage[title]{appendix}

\usepackage{caption}
\usepackage[bottom,hang,flushmargin]{footmisc}
\usepackage{comment}
\usepackage[bottom]{footmisc}
\usepackage{caption}
\usepackage{subcaption}
\usepackage{enumitem}

\newcommand{\bi}{\begin{itemize}}
\newcommand{\ei}{\end{itemize}}
\newcommand{\bal}{\begin{align}}
\newcommand{\eal}{\end{align}}

\newcommand{\EE}{\mathbb{E}}
\newcommand{\PP}{\mathbb{P}}

\newcommand{\QQ}{\mathbb{Q}}

\newcommand{\bX}{\mathbf{X}}

\newcommand{\bY}{\mathbf{Y}}

\newcommand{\bW}{\mathbf{W}}

\newcommand{\by}{\mathbf{y}}

\newcommand{\bv}{\mathbf{v}}
\newcommand{\bc}{\mathbf{c}}
\newcommand{\bu}{\mathbf{u}}

\newcommand{\bA}{\mathbf{A}}

\newcommand{\bU}{\mathbf{U}}

\newcommand{\bS}{\mathbf{S}}

\newcommand{\bI}{\mathbf{I}}

\newcommand{\bmu}{\mathbf{\mu}}

\newcommand{\bH}{\mathbf{H}}

\newcommand{\bK}{\mathbf{K}}

\newcommand{\bG}{\mathbf{G}}
\newcommand{\bh}{\mathbf{h}}
\newcommand{\bg}{\mathbf{g}}
\newcommand{\tbG}{\tilde{\mathbf{G}}}
\newcommand{\tG}{\tilde{G}}
\newcommand{\hbG}{\hat{\bG}}

\newcommand{\cN}{\mathcal{N}}

\newcommand{\cO}{\mathcal{O}}

\newcommand{\cD}{\mathcal{D}}
\newcommand{\cS}{\mathcal{S}}

\newcommand{\cG}{\mathcal{G}}

\newcommand{\tY}{\tilde{Y}}

\newcommand{\tO}{\tilde{\cO}}

\newcommand{\hY}{\hat{Y}}

\newcommand{\hby}{\hat{\by}}

\newcommand{\bSigma}{\mathbf{\Sigma}}

\newcommand{\tbK}{\tilde{\mathbf{K}}}
\newcommand{\Tr}{\text{Tr}}

\newcommand{\grad}{\bigtriangledown}

\newcommand{\range}{\text{range}}

\def\<{\langle}
\def\>{\rangle}

\newtheorem{theorem}{\textbf{Theorem}}
\newtheorem{lemma}{\textbf{Lemma}}

\newtheorem{definition}{\textbf{Definition}}

\newtheorem{proof}{\textbf{Proof}}

\title{{\huge Understanding GANs: the LQG Setting}}

\author{Soheil Feizi$^1$, Farzan Farnia$^2$,  Tony Ginart$^2$  and David Tse$^2$\\\\
$^1$ Department of Computer Science, University of Maryland, College Park\\
$^2$ Department of Electrical Engineering, Stanford University}
\date{}

\begin{document}
\maketitle

\begin{abstract}
Generative Adversarial Networks (GANs) have become a popular method to learn a probability model from data. In this  paper, we aim to provide an understanding of some of the basic issues surrounding GANs including their formulation, generalization and stability on a simple benchmark where the data has a high-dimensional Gaussian distribution. Even in this simple benchmark, the GAN problem has not been well-understood as we observe that existing state-of-the-art GAN architectures may fail to learn a proper generative distribution owing to (1) stability issues (i.e., convergence to bad local solutions or not converging at all), (2) approximation issues (i.e., having improper global GAN optimizers caused by inappropriate GAN's loss functions), and (3) generalizability issues (i.e., requiring large number of samples for training). In this setup, we propose a GAN architecture which recovers the maximum-likelihood solution and demonstrates fast generalization. Moreover, we analyze global stability of different computational approaches for the proposed GAN optimization and highlight their pros and cons. Finally, we outline an extension of our model-based approach to design GANs in more complex setups than the considered Gaussian benchmark. 
\end{abstract}

\section{Introduction}\label{sec:intro}
Learning a probability model from data is a fundamental problem in statistics and machine learning. Building off the success of deep learning, Generative Adversarial Networks (GANs) \cite{goodfellow2014generative} have given this age-old problem a face-lift.
In contrast to traditional methods of parameter fitting like maximum likelihood estimation, the GAN approach views the problem as a {\em game}  between a {\it generator} whose goal is to generate fake samples that are close to the real data training samples and a {\it discriminator} whose goal is to distinguish between the real and fake samples. The generator and the discriminator are typically implemented by deep neural networks. GANs have achieved impressive performance in several domains (e.g., \cite{ledig2016photo,reed2016generative}). However, training good GANs is still challenging and it is an active area to design GANs with better and more stable performance (e.g., \cite{arjovsky2017wasserstein,gulrajani2017improved,sanjabi2018solving} and Section \ref{sec:prior-work}).

GANs are typically designed without any explicit modeling of the data. Are they {\em universal} learning algorithms, i.e. can they learn a very wide range of data distributions? If not, what are their limits? Can better GANs be designed if we use an explicit model of the data? These are the questions we wish to explore in this paper.

GANs' evaluations are primarily done on real data, typically images. Although clearly valuable, such evaluations are often subjective owing to not having clear baselines to compare against. To make progress on the above questions, we report instead experiments of state-of-the-art GANs on {\em synthetic} data where clear baselines are known. We chose one of the simplest high-dimensional distributions: the Gaussian distribution.

\begin{figure}[t]
\centering
\includegraphics[width=0.8\linewidth]{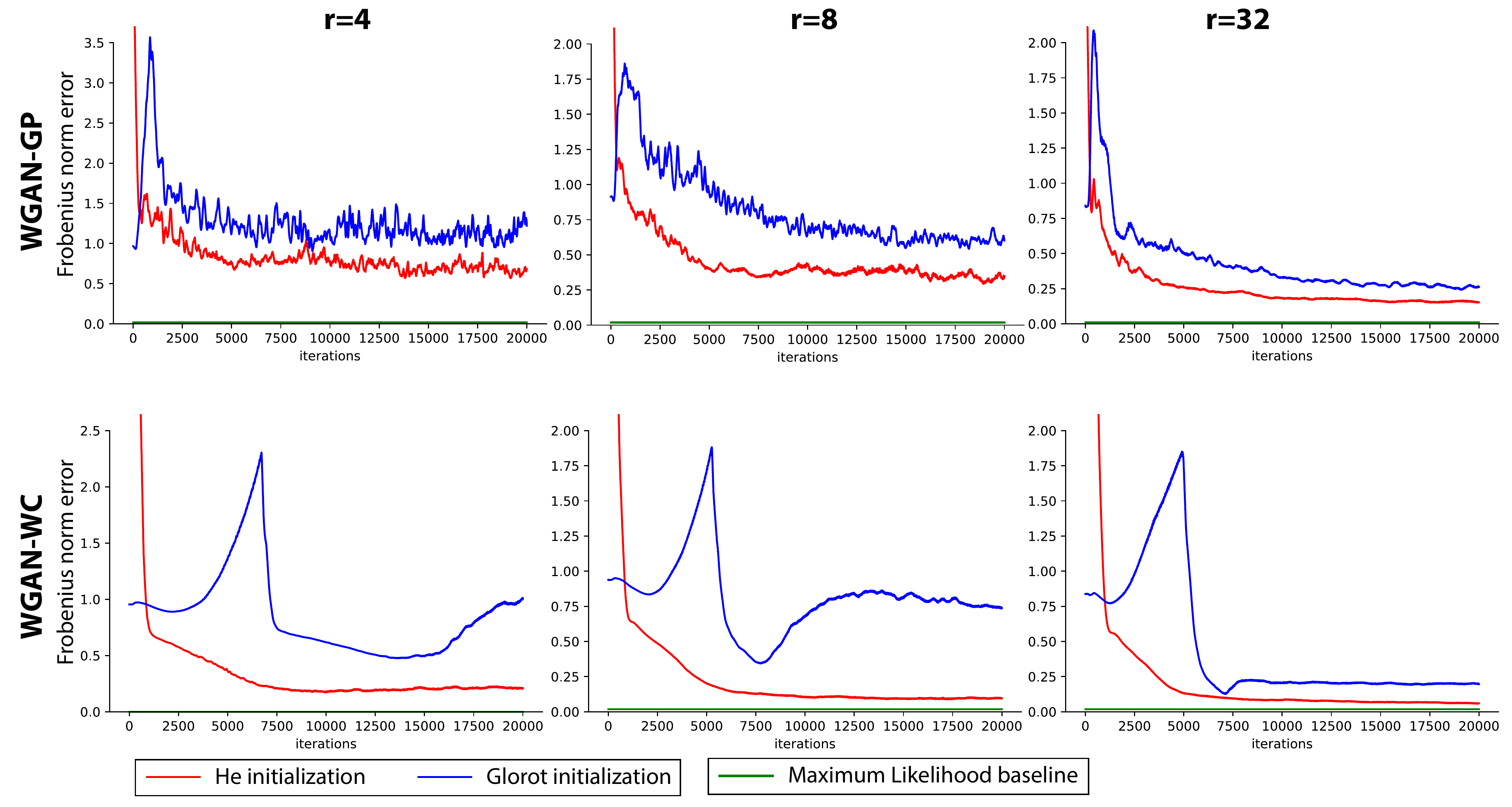}
\caption{An illustration of the performance of WGAN+GP and WGAN+WC in different values of $r$  (the dimension of the input randomness to the generator) with different initialization procedures when the generator and the discriminator functions are both neural networks.}
\label{fig:nonlinearG}
\end{figure}

\subsection{Experiments}
\begin{figure}[t]
\centering
\includegraphics[width=0.8\linewidth]{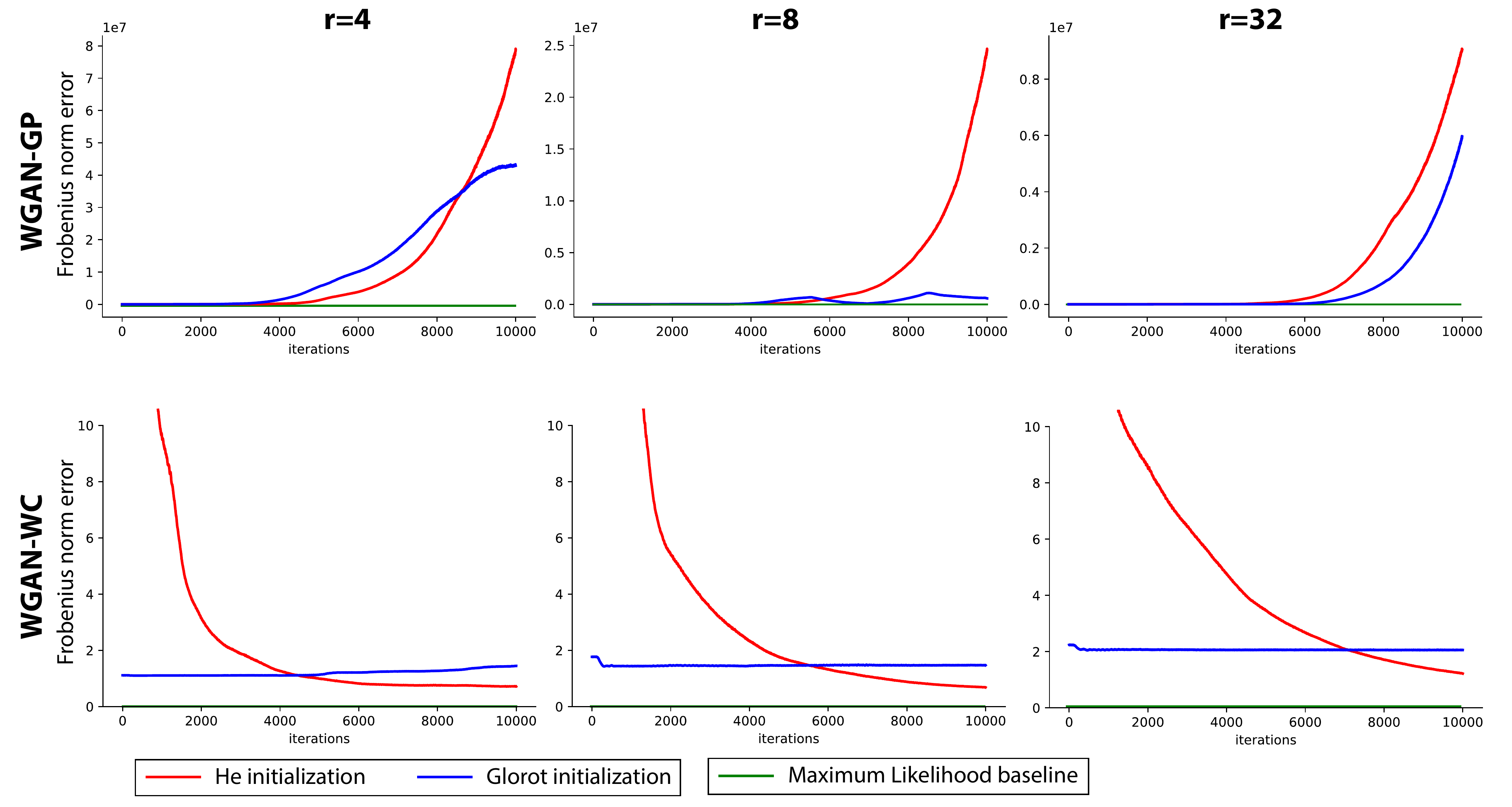}
\caption{A repeat of experiments of Figure \ref{fig:nonlinearG} using the ELU activation function instead of the ReLU activation. GAN's performance changes drastically.  Note the  range of the $y$-axes in the plots for WGAN-GP.}
\label{fig:ELU}
\end{figure}

In our first set of experiments, we generate $n=100,000$ samples from a $d=32$ dimensional Gaussian distribution $\cN(\mathbf{0},\bK)$ when $\bK$ is the normalized identity matrix; $\bK=\bI/\sqrt{d}$. We train two state-of-the-art GAN architectures in our experiments; WGAN+Weight Clipping (WGAN+WC) \cite{arjovsky2017wasserstein} and WGAN+Gradient Penalty (WGAN+GP) \cite{gulrajani2017improved}. We use the neural net generator and discriminator  with hyper-parameter settings as recommended in \cite{gulrajani2017improved}.  Each of the neural networks has three hidden layers, each with 64 neurons and ReLU activation functions. To evaluate GAN's performance, we compute the Frobenius norm between covariance matrices of observed and generative distributions.

Figure \ref{fig:nonlinearG} shows the performance of GANs for various values of $r$, the dimension of the randomness (i.e., input to the generator, for which we use the standard Gaussian randomness ) and for two random initializations for ReLU layers using the standard {\it He} \cite{he2015delving} and {\it Glorot} \cite{glorot2010understanding} procedures. In these experiments, we observe two types of instability in GAN's performance; oscillating behaviour (e.g., WGAN-GP, $r=4,8$) and convergence to different and bad local solutions. Even after $20,000$ training epochs, the error does not approach zero in most cases. We observe similar trends when we use a random covariance instead of the normalized identity matrix (Appendix Figure \ref{fig:nonlinearG-random}).

Next, we repeat these experiments using the ELU activation function \cite{clevert2015fast} instead of the ReLU activation. Since ELU activation is differentiable everywhere, we thought it may provide more stability. However, surprisingly GAN's performance becomes dramatically worst (Figure \ref{fig:ELU}). Reference \cite{gulrajani2017improved} has noticed similar behavior in another setting as well, but there was no explanation, and  we are still surprised that such a drastic phenomenon can happen even in such a simple Gaussian setting.

In our next set of experiments, we attempt to improve performance by restricting the generator to be linear, since both the observed data and the randomness come from Gaussian distributions. (The discriminator is still the ReLU neural network). Since the generator is linear, zero error cannot be achieved in the case of $r<d$. In this case, a natural baseline is the $r$-PCA of the sample covariance matrix (Appendix Section \ref{SMsec:exp}). GAN's performance improves compared to the case of nonlinear generator (Figure \ref{fig:linearG}). We do not observe oscillating behavior in WGAN+GP. However, we still observe convergence to different bad local solutions for both WGAN+GP and WGAN+WC. Unlike Figure \ref{fig:nonlinearG} where WGAN+WC was performing better than WGAN+GP, here the performance of WGAN+WC is significantly worst than that of the WGAN+GP. Also, unlike other cases, in WGAN+GP when $r=4$, the Glorot initialization achieves a smaller error than that of the He initialization. These results highlight sensitivity of state-of-the-art GANs even in a simple benchmark.

\begin{figure}[t]
\centering
\includegraphics[width=0.8\linewidth]{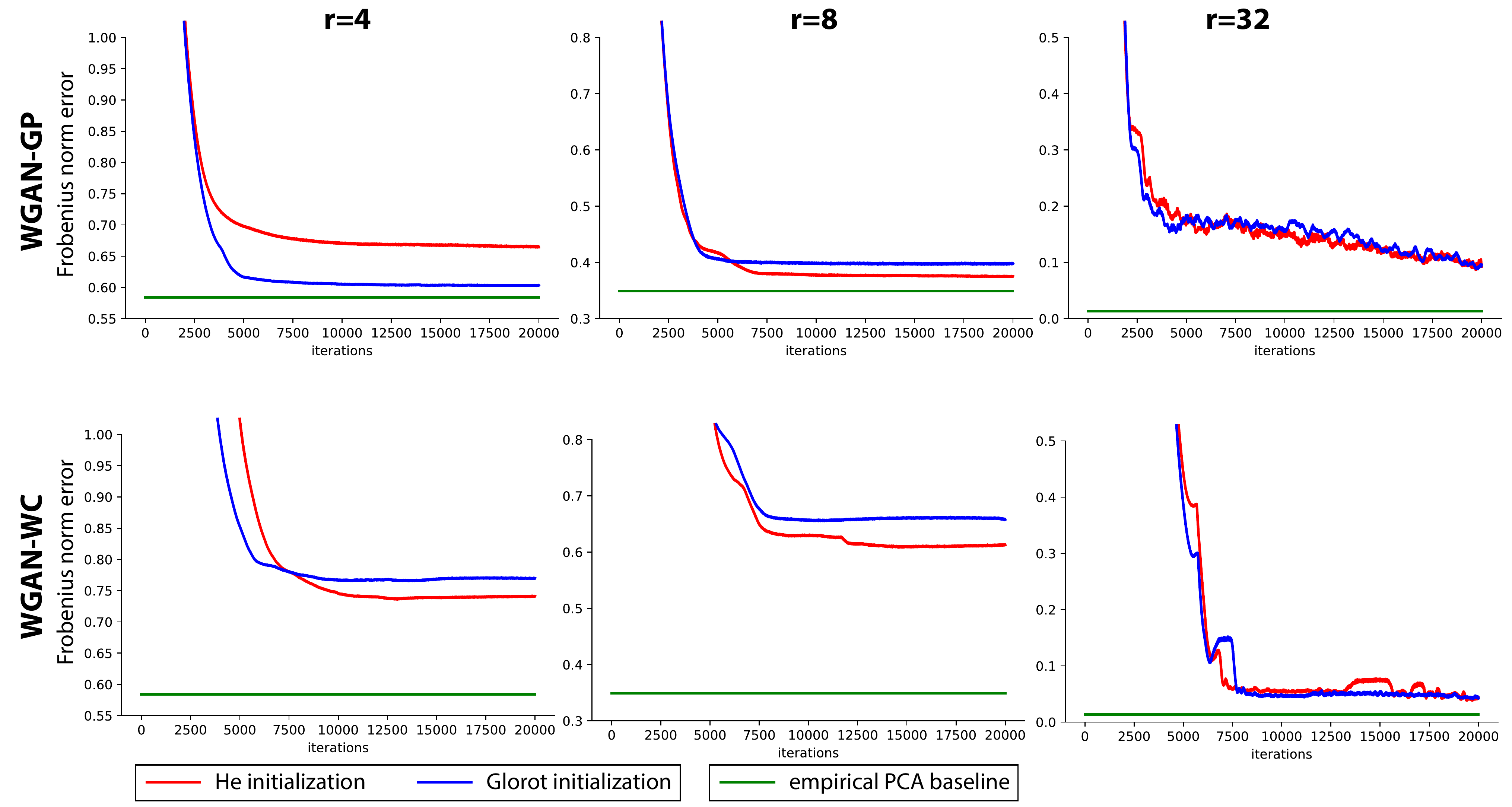}
\caption{A repeat of experiments of Figure \ref{fig:nonlinearG} when the generator function is linear. A random covariance matrix is chosen instead of the identity matrix.}
\label{fig:linearG}
\end{figure}

\subsection{Quadratic GAN}

The performance of state-of-the-art GANs on Gaussian data gives some evidence on the limitations of a model-free approach to the GAN design. We ask: is there is a natural GAN architecture for learning Gaussian data? We ask this question not because we are in desperate need of a new way of learning Gaussian distributions, but because we want to see how much gain a model-based approach can buy on this simple benchmark as well as hope to learn something in the process about designing GANs in general.


Figure \ref{fig:quadratic_gan}-a is our proposed GAN architecture for the Gaussian benchmark \footnote{For simplicity, we assume that samples have been centered to have zero means. In the general case, the generator should be an affine function.}. We refer to this architecture as {\it Quadratic GAN}.  Figure \ref{fig:quadratic_gan}-b compares performance of quadratic GAN and WGAN+GP for $r=32$. Quadratic GAN demonstrates stable behavior and much faster convergence to the maximum-likelihood baseline compared to WGAN. In fact, due to its simple structure, training of the Quadratic GAN takes less than $1$ second on a laptop CPU which is orders of magnitudes faster than training WGAN on a GPU.

We designed the Quadratic GAN in three steps: First, we formulated GAN's objective by specifying the appropriate loss  to naturally match the Gaussian model for the data  (Section \ref{sec:sec-formulation}). This allows us to show that the global population solution of the minmax problem is the $r$-PCA of the (true) covariance matrix of the Gaussian model (Theorem \ref{thm:PCA}). However, this initial architecture can have poor generalization performance (Appendix Section \ref{sec:generalization}). Second, we further constrained the discriminator to keep the good optimal solution of the population-optimal architecture while enabling fast generalization (Section \ref{sec:qgan}). We refer to this architecture as the quadratic GAN (Figure \ref{fig:quadratic_gan}). We show that the global optimizer of quadratic GAN applied on the empirical distribution is the {\em empirical} $r$-PCA (Theorem \ref{thm:eq-quality}). Finally, we study the global stability of different computational approaches for solving the proposed GAN architecture. In particular, we prove that in the full-rank case alternating gradient descent converges globally to the minmax solution, under some conditions (Section \ref{sec:stability}). 

\begin{figure}[t]
\centering
\includegraphics[width=0.9\linewidth]{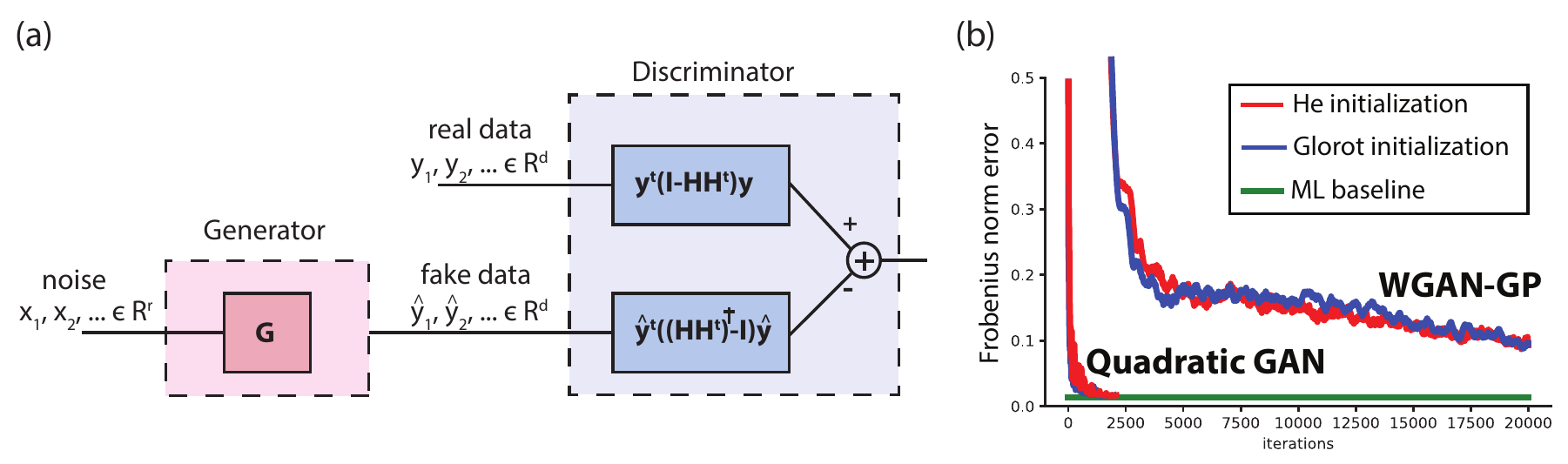}
\caption{(a) Quadratic GAN, with a linear generator and a quadratic discriminator. On the training data, the generator minimizes over the $d$ by $r$ matrix $G$  and the adversary maximizes over the $d$ by $d$ matrix $H$. (b) Performance comparison between quadratic GAN and WGAN+GP for $r=32$. }
\label{fig:quadratic_gan}
\end{figure}

\subsection{Prior Work}\label{sec:prior-work}
Broadly speaking, previous work in GANs study three main properties: (1) Stability where the focus is on the convergence of the commonly used alternating gradient descent approach to global/local optimizers (equilibriums) for GAN's optimization (e.g., \cite{nagarajan2017gradient,roth2017stabilizing,mescheder2017numerics,sanjabi2018solving,daskalakis2017training}, etc.), (2) Formulation where the focus is on designing proper loss functions for GAN's optimization (e.g., WGAN+Weight Clipping \cite{arjovsky2017wasserstein}, WGAN+Gradient Penalty \cite{gulrajani2017improved}, GAN+Spectral Normalization \cite{miyato2018spectral}, WGAN+Truncated Gradient Penalty \cite{petzka2017regularization}, relaxed WGAN \cite{guo2017relaxed}, $f$-GAN \cite{nowozin2016f}, MMD-GAN \cite{dziugaite2015training,li2015generative} , Least-Squares GAN \cite{mao2016multi}, Boundary equilibrium GAN \cite{berthelot2017began}, etc.), and (3) Generalization where the focus is on understanding the required number of samples to learn a probability model using GANs (e.g., \cite{arora2017generalization}). We address all three issues in the design of the Quadratic GAN.
Some references have also proposed model-based GANs for the Gaussian benchmark (\cite{nagarajan2017gradient,daskalakis2017training}). For example, \cite{daskalakis2017training} uses a quadratic function as the discriminator in the WGAN optimization. This design, however, does not recover the maximum likelihood/PCA solutions in the Gaussian benchmark, unlike the  Quadratic GAN. Moreover, no global stability results were proven.

\section{A General Formulation for GANs}\label{sec:sec-formulation}
Let $\{\by_i\}_{i=1}^{n}$ be $n$ observed data points in $\mathbb{R}^d$ drawn i.i.d. from the distribution $\PP_{Y}$. Let $\QQ_Y^n$ be the empirical distribution of these observed samples. Moreover, let $\PP_X$ be a normal distribution $\cN(\mathbf{0},\bI_r)$. GANs can be viewed as an optimization that minimizes a distance between the observed empirical distribution $\QQ_Y^n$ and the generated distribution $\PP_{\bG(X)}$. The {\it population} GAN optimization replaces $\QQ_Y^n$ with $\PP_Y$ and is the setting we focus on in this section. The question we ask in this section is: what is a natural way of specifying a loss function $\ell$ for GANs and how it determines the GAN's objective? We answer the question in general and then specialize to the Gaussian benchmark  by choosing an appropriate loss function for that case. We then show that we can get a good population solution under this loss function.

\subsection{WGAN Revisited}
Let us start with the WGAN optimization \cite{arjovsky2017wasserstein}:
\begin{align}\label{opt:w1-GAN}
\min_{\bG(.)\in\cG}~ W_1(\PP_Y,\PP_{\bG(X)}),
\end{align}
where $\cG$ is the set of generator functions, and the $p$-th order Wasserstein distance between distributions $\PP_{Z_1}$ and $\PP_{Z_2}$ is defined as \cite{villani2008optimal}
\begin{align}\label{eq:W2}
W_p^p(\PP_{Z_1},\PP_{Z_2}):=\min_{\PP_{Z_1,Z_2}} \EE\left[\|Z_1-Z_2\|^p\right],
\end{align}
where the minimization is over all joint distributions with marginals fixed. Replacing \eqref{eq:W2} in \eqref{opt:w1-GAN}, the WGAN optimization can be re-written as
\begin{align}\label{opt:unsup-emp-w1}
\min_{\bG(.)\in\cG}~ \min_{\PP_{\bG(X),Y}}~&\EE\left[\|Y-\bG(X)\|\right].
\end{align}
or equivalently:
\begin{align}\label{opt:unsup-emp-w1_2}
\min_{\PP_{X,Y}}~\min_{\bG(.)\in\cG}~&\EE\left[\|Y-\bG(X)\|\right],
\end{align}
where the  minimization is over all joint distributions $\PP_{X,Y}$ with fixed marginals $\PP_X$ and $\PP_Y$.

We now connect (\ref{opt:unsup-emp-w1_2}) to the {\em supervised learning} setup.
In supervised learning, the joint distribution  $\PP_{X,Y}$ is fixed and the goal is to learn a relationship between the feature variable represented by $X\in\mathbb{R}^{r}$, and the target variable represented by $Y\in\mathbb{R}^{d}$, according to the following optimization:
\begin{align}\label{opt:sup-emp}
\min_{\bG(.)\in \cG}~&\EE\left[\ell\left(Y,\bG(X)\right)\right],
\end{align}
where $\ell $ is the {\it loss} function. Assuming the marginal distribution of $X$ is the same in both optimizations \eqref{opt:unsup-emp-w1_2} and \eqref{opt:sup-emp}, we can connect the two optimization problems by setting $\ell(y,y')=\|y-y'\|$ in optimization \eqref{opt:sup-emp}. Note that for every fixed $\PP_{X,Y}$, the solution of the supervised learning problem (\ref{opt:sup-emp}) yields a predictor $g$ which is a feasible solution to the WGAN optimization problem \eqref{opt:unsup-emp-w1_2}.
Therefore, the WGAN optimization \eqref{opt:unsup-emp-w1} can be re-interpreted as solving the {\em easiest} such supervised learning problem, over all possible joint distributions $\PP_{X,Y}$ with fixed $\PP_X$ and $\PP_Y$.

\subsection{From Supervised to Unsupervised Learning}

GAN is a solution to an unsupervised learning problem. What we are establishing above is a general connection between supervised and unsupervised learning problems: a good predictor $\bG$ learnt in a supervised learning problem can be used to generate samples of the target  variable Y. Hence, to solve an unsupervised learning problem for $Y$ with distribution $\PP_Y$, one should solve the easiest supervised learning problem $\PP_{X,Y}$ with given  marginal $\PP_Y$ (and $\PP_X$, the randomness generating distribution).
This is in contrast to the traditional view of the unsupervised learning problem as observing the feature variable $X$ without the label $Y$. (Thus in this paper we break with tradition and use $Y$ to denote data  and $X$ as randomness for the generator in stating the GAN problem.)

This connection between supervised and unsupervised learning leads to a natural way of specifying the loss function in GANs: we simply replace the $\ell_2$ in (\ref{opt:unsup-emp-w1}) with a general loss function $\ell$:
\begin{align}\label{opt:gan-gen-loss}
\min_{\bG(.)\in\cG}~ \min_{\PP_{\bG(X),Y}}~&\EE\left[\ell\left(Y,\bG(X)\right)\right].
\end{align}
The inner optimization is the optimal transport problem between distributions of $\bG(X)$ and $Y$ \cite{villani2008optimal} with general cost $\ell$. This is a linear programming problem for general cost, so there is always a dual formulation (the Kantorovich dual \cite{villani2008optimal}). The dual formulation can be interpreted as a generalized discriminator optimization problem for the cost $\ell$. (For example, in the case of $\ell$ being the Euclidean norm, we get WGAN.)  Hence, we use (\ref{opt:gan-gen-loss}) as a formulation of GANs for general loss functions.

Note that an optimal transport view to GANs has been studied in other references (e.g., \cite{arjovsky2017wasserstein,liu2017approximation}). Our contribution in this section is to make a {\it connection} between supervised and unsupervised learning problems which we will exploit to specify a proper loss function for GANs in the Gaussian model.

\subsection{Quadratic Loss and Linear Generators}

The most widely used loss function in supervised learning is the quadratic loss: $\ell(y,y')=\|y-y'\|^2$ ({\em squared} Euclidean norm).  The quadratic loss has a strong synergy with the Gaussian model, as observed by Gauss himself. For example, under the Gaussian model and the quadratic loss in the supervised learning problem (\ref{opt:sup-emp}), the optimal $g$ is linear, thus forming a statistical basis for linear regression. Given the connection between supervised and unsupervised learning, we use this loss function for formulating the GAN for Gaussian data  .   This choice of the loss function leads to the following GAN optimization which we refer to as {\it W2GAN}:
\begin{align}\label{opt:w2gan}
\min_{\bG(.)\in\cG}~ W_2^2 (\PP_Y,\PP_{\bG(X)}).
\end{align}
A natural choice of $\cG$ is the set of all linear generators, from $\mathbb{R}^{r}$ to $\mathbb{R}^{d}$.

Since Wasserstein distances are weakly continuous measures in the probability space \cite{villani2008optimal}, similar to WGAN, the optimization of the W2GAN is well-defined even if $r<d$. The dual formulation (discriminator) for $W_2^2$ is  \cite{villani2008optimal}:
\begin{align}\label{opt:dual-w2}
W_2^2(\PP_Y,\PP_{\bG(X)})=\max_{\psi(.) \text{:convex}}&~ \EE\left[\|Y\|^2-2\psi(Y)\right]-\EE\left[2\psi^*(\bG(X))-\|\bG(X)\|^2\right],
\end{align}
where
\begin{align}\label{opt:conjugate}
\psi^*(\hby):=\max_{\bv} \left(\bv^t \hby-\psi(\bv)\right)
\end{align}
is the convex-conjugate of the function $\psi(.)$. Combining (\ref{opt:w2gan} )and (\ref{opt:dual-w2}), we obtain the minmax formulation of W2GAN:
\begin{align}
\label{eq:minmax_w2gan}
\min_{g \in \cG} \max_{\psi(.) \text{:convex}}&~ \EE\left[\|Y\|^2-2\psi(Y)\right]-\EE\left[2\psi^*(\bG(X))-\|\bG(X)\|^2\right].
\end{align}

\subsection{Population Solution: PCA}
There is a simple solution to the optimization problem (\ref{opt:w2gan}) in the population setting.

\begin{theorem}\label{thm:PCA}
Let $Y\sim \cN(\mathbf{0},\bK)$ where $\bK$ is full-rank. Let $X\sim \cN(\mathbf{0},\bI_r)$ where $r\leq d$. The optimal GAN solution in the population setting under linear generators $\cG$ is the $r$-PCA solution of $Y$.
\end{theorem}

We say $\hY$ is the $r$-PCA solution of $Y$ if $\bK_{\hY}$ is a rank $r$ matrix whose top $r$ eigenvalues and eigenvectors are the same as top $r$ eigenvalues and eigenvectors of $\bK$. This theorem is satisfactory as it connects GANs to PCA, one of the most basic unsupervised learning methods.

\section{Quadratic GAN} \label{sec:qgan}

The discriminator of the W2GAN optimization \eqref{eq:minmax_w2gan} is constrained over all convex functions. Since this set is non-parametric, we are unable to use gradient descent to compute a solution for this optimization. Moreover, having such a large feasible set for the discriminator function can cause poor {\em generalization}.

Consider the empirical version of the population W2GAN optimization problem (\ref{opt:w2gan}):
\begin{align}\label{opt:w2gan_emp}
\min_{\bG(.)\in\cG}~ W_2^2 (\QQ^n_Y,\PP_{\bG(X)}),
\end{align}
where $\QQ^n_Y$ is the empirical distribution of the $n$ data points $\{\by_i\}_{i=1}^{n}$. Let $g^*_n$ be the optimal solution of this problem. The distance between the generated distribution $\bG^*(X)$ and the true distribution $\PP_Y$, $W^2_2(\PP_Y, \PP_{g^*_n(X)})$, converges to zero as $n \rightarrow \infty$.  It was shown in \cite{arora2017generalization} that if the generator class $\cG$ is rich enough so that the generator can memorize the data and generate the empirical distribution $\QQ_Y^n$ itself, then this rate of convergence is very slow, of the order of $n^{-2/d}$. (Strictly speaking, they have only shown it for the $W_1$ distance, but a very similar result holds for $W_2$ as well.) This is because the empirical distribution $\QQ_Y^n$ converges very slowly to the true distribution $\PP_Y$ in the $W_2$ distance. Hence, the number of samples required for convergence is exponential in the dimension $d$. In Appendix Section \ref{sec:generalization}, we show that in our Gaussian setup, even if we constrain the generators to single-parameter linear functions that can generate the true distribution, the rate of convergence is still $n^{-2/d}$. 

Therefore, to overcome the generalization issue, the only option is to further constrain the {\em discriminator}.
Ideally one would like to {\it properly} constrain the discriminator function such that any population solution of the constrained optimization is a population solution of the original optimization and vice versa, while at the same time allowing fast generalization. In this section, we show how we can achieve this goal for the Gaussian benchmark. This view can potentially be extended to more complex distributions as we explain in Section \ref{sec:beyond}.

The following lemma characterizes the optimal solution of optimization \eqref{opt:dual-w2} \cite{chernozhukov2017monge}:
\begin{lemma}\label{lem:det-coupling}
Let $\PP_Y$ be absolutely continuous whose support contained in a convex set in $\mathbb{R}^d$. For a fixed $\bG(.)\in\cG$, let $\psi^{opt}$ be the optimal solution of optimization \eqref{opt:dual-w2}. This solution is unique. Moreover, we have
\begin{align}\label{eq:deterministic-mapping}
\grad \psi^{opt}(Y)\stackrel{\text{dist}}{=}\bG(X),
\end{align}
where $\stackrel{\text{dist}}{=}$ means matching distributions.
\end{lemma}
In our benchmark setup, since $\bG(X)$ is Gaussian, $\grad \psi^{opt}$ is a linear function. Thus, without loss of generality, $\psi(.)$ in the discriminator optimization can be constrained to quadratic functions of the form $\psi(\by)=\by^t \bA \by/2$ where $\bA$ is positive semidefinite. For the quadratic function, we have $\psi^*(\hby)=\hby^t \bA^{\dagger} \hby/2$ when $\range(\hY)\subseteq \range(\bA)$.

Replacing these in optimization \eqref{eq:minmax_w2gan}, we obtain:
\begin{align}\label{opt:quadratic-GAN}
\min_{\bG}~\max_{\bA \succeq 0}~ &\EE\left[ Y^t (\bI-\bA)Y \right]- \EE\left[ \hY^t (\bA^{\dagger}-\bI) \hY \right]\\
&\range(\bG)\subseteq \range(\bA).\nonumber
\end{align}
Without loss of generality, we can replace the constraint $\range(\bG)\subseteq \range(\bA)$ with $\range(\bG)= \range(\bA)$. It is because for a given $\bA$, this increases the size of the feasible set for $\bG$ optimization, thus the objective can achieve a smaller value. For a given $\bG$, one can decompose $\bA$ as $\bA_{1}+\bA_{2}$ where $\range(\bA_1)=\range(\bG)$ and $\range(\bA_2)\cap \range(\bG)=\emptyset$. Note that by ignoring $\bA_2$, the objective function does not decrease. Therefore, optimization \eqref{opt:quadratic-GAN} can be written as
\begin{align}\label{opt:quadratic-GAN-eq}
\min_{\bG}~\max_{\bA \succeq 0}~ &\EE\left[ Y^t (\bI-\bA)Y \right]- \EE\left[ \hY^t (\bA^{\dagger}-\bI) \hY \right]\\
&\range(\bG)= \range(\bA).\nonumber
\end{align}
Using the fact that trace is invariant under cyclic permutations and by replacing $\bA=\bH \bH^t$, the objective function of the above optimization can be re-written as:
\begin{align}\label{eq:quadratic-gan-objective}
J(\bG,\bH)=&\Tr\left[ \left(\bI-\bH\bH^t\right) \bK \right]- \Tr\left[\left(\left(\bH\bH^t\right)^{\dagger}-\bI\right) \bG \bG^t \right].
\end{align}
In practice, we apply GANs to the observed data (i.e., the empirical distribution). In that case, in the above objective function, $\bK$ (the true covariance) should be replaced by $\tbK$ (the empirical covariance). This leads to the {\em quadratic GAN} optimization:
\begin{align}\label{opt:quadratic-gan-empirical}
\min_{\bG}~\max_{\bH}~ &\Tr\left[ \left(\bI-\bH\bH^t\right) \tbK \right]- \Tr\left[\left(\left(\bH\bH^t\right)^{\dagger}-\bI\right) \bG \bG^t \right]\\
&\range(\bG)= \range(\bH).\nonumber
\end{align}
Note that since the global optimizer of optimization \eqref{eq:quadratic-gan-objective} is PCA (Theorem \ref{thm:PCA}), the global optimizer of optimization \eqref{opt:quadratic-gan-empirical} is empirical PCA:

\begin{theorem}\label{thm:eq-quality}
Let $\tilde{\bK}_r$ be the $r$-PCA of the sample covariance matrix. Let $(\bG^*,\bH^*)$ be a global solution for the quadratic GAN optimization \eqref{opt:quadratic-gan-empirical}. Then, we have $\bG^* (\bG^*)^t=\tilde{\bK}_r$. I.e., quadratic GAN recovers the empirical PCA solution as the generative model.
\end{theorem}

Next, we examine the generalization error of the quadratic GAN. Consider the case where $d=r$ (the case $r<d$ is similar). The generalization error can be written as the $W_2$ distance between the true distribution $\PP_Y$ and the learned distribution $\PP_{\bG^*(X)}$ (Appendix Section \ref{sec:generalization}):
\begin{align}
W_2^2\left(\PP_{Y},\PP_{\bG^*(X)}\right) = W_2^2\left(\cN(0,\bK),\cN(0,\tbK)\right).
\end{align}
The $W_2^2$ distance between two Gaussians depends only on the covariance matrices. More specifically:
\begin{align}
W_2^2\left(\cN(0,\bK),\cN(0,\tbK)\right) &=\Tr(\bK)+\Tr(\tbK)-2\Tr\left(\left(\bK^{1/2}\tbK \bK^{1/2}\right)^{1/2}\right).
\end{align}
Hence, the convergence of this quantity only depends on the convergence of the empirical covariance to the population covariance, together with smoothness property of this function of the covariance matrices. The convergence has been established to be at a quick rate of $\tO(\sqrt{d/n})$ \cite{rippl2016limit}.

\section{Stability}\label{sec:stability}
Theorem \ref{thm:eq-quality} merely focuses on the quality of the global solution of the quadratic GAN's optimization, ignoring its computational aspects. One common way to solve the GAN's min-max optimization is to use alternating gradient descent with $s_G$ gradient steps for the generator updates and $s_D$ gradient steps for the discriminator updates. For simplicity, we refer to such a method as the $(s_G,s_D)$-alternating gradient descent. In this section, we analyze the global stability of the quadratic GAN under the alternating gradient descent approach.

First, we analyze the stability of the quadratic GAN under the $(1,1)$-alternating GD in the full-rank case. By using variables $\bU:=\bG \bG^t$ and $\bA:=\bH\bH^t$, optimization \eqref{opt:quadratic-gan-empirical} can be written as
\begin{align}\label{opt:quadratic-gan-empirical-AU}
\min_{\bU}~\max_{\bA}~ &\Tr\left[ \left(\bI-\bA\right) \tbK \right]- \Tr\left[\left(\left(\bA^t\right)^{\dagger}-\bI\right) \bU^t \right]\\
&\range(\bU)= \range(\bA).\nonumber
\end{align}
For this case, we have the following result:
\begin{theorem}\label{thm:stability-1-1-AU}
In the quadratic GAN optimization \eqref{opt:quadratic-gan-empirical-AU}, assuming full rank $\bA$ and $r=d$, the $(1,1)$-alternating gradient descent is globally stable.
\end{theorem}
\begin{proof}
Optimization \eqref{thm:stability-1-1-AU} is a convex-concave min-max problem. Using the Arrow-Hurwicz-Uzawa result \cite{freund1996game}, one can show that the $(1,1)$-alternating gradient descent is globally stable for this optimization. 
\end{proof}
In the standard quadratic GAN, the alternating GD is applied on the $(\bG,\bH)$ objective function which is not generally convex-concave. For this case, we have the following result:
\begin{theorem}\label{thm:stability-1-1}
In the quadratic GAN optimization \eqref{opt:quadratic-gan-empirical}, assuming $\tbK=\bI$, full rank $\bH$ and $r=d$, the $(1,1)$-alternating gradient descent is globally stable.
\end{theorem}
To prove Theorem \ref{thm:stability-1-1}, we use the following function as a Lyapunov function:
\begin{align}\label{eq:lyp}
V(\bG,\bH)=&\Tr\left[\bG\bG^t-\bI-\log\left(\bG\bG^t\right) \right]+\Tr\left[\bH\bH^t-\bI-\log\left(\bH\bH^t\right) \right].
\end{align}
Each term of this function is the Von Neumann divergence. We prove that this non-negative function is monotonically decreasing along every trajectory of the (1,1)-alternating gradient descent and its value is zero at the global solution. This phenomena is non-trivial because the Frobenius norm distance between $\bG\bG^t$ and $\tbK$ is not monotonically decreasing along every trajectory (Appendix Figure \ref{fig:lyp}).

In the low-rank case where $r<d$, however, we have the following negative result:

\begin{theorem}\label{thm:instability-1-1}
In the quadratic GAN optimization \eqref{opt:quadratic-gan-empirical}, if $r<d$, the $(s_G,s_D)$-alternating gradient descent is {\it not} globally stable for any $s_G$ and $s_D$.
\end{theorem}
\begin{proof}
Note that due to the constraint $\range(\bG)=\range(\bH)$, by initializing $\bG$ to some matrix, the colum-space of $\bG$ and $\bH$ does not change with gradient updates. This leads to the above result.
\end{proof}
One can think about using an equivalent optimization \eqref{opt:quadratic-gan-empirical} where the constraint $\range(\bG)=\range(\bH)$ is replaced by the constraint $\range(\bG)\subseteq \range(\bH)$ (by assuming $\bA=\bH \bH^t$). For example, if $\bH$ is full-rank, this constraint always holds. However, this does not solve the stability issue of Theorem \ref{thm:instability-1-1} . It is because in the desired saddle point, $\bH^*$ should be a low-rank matrix whose range matches the range of $\bG^*$. If one starts the alternating GD with a full-rank $\bH$, the second term of the objective function \eqref{opt:quadratic-gan-empirical} would decrease unboundedly when $\bH$ loses rank in the null-space of $\bG$ (because of the term $(\bH\bH^t)^{\dagger}\bG\bG^t$). Therefore, unless $\bH$ has a matching range with $\bG$, alternating GD will not converge to a low-rank solution for $\bH$. 

As we explained above, the main source of the instability of the quadratic GAN optimization in the low-rank case comes from the constraint $\range(\bG)=\range(\bH)$, i.e. the matching column-space of the generator and the discriminator functions. One way to deal with this issue is to decouple the optimization to two parts where in one part we optimize the subspace and in the second part, we solve GAN's min-max optimization {\it within} that subspace. Below, we explain this approach. We denote the subspace by some orthogonal basis $\bS \in \mathbb{R}^{d\times r}$ where $\bS^t\bS=\bI$. Then, we re-write
\begin{align}\label{eq:mappings}
\bG:=\bS \bG_S, \quad \bH:=\bS \bH_S,
\end{align}
where $\bG_S$ and $\bH_S$ are full-rank $r \times r$ matrices. Also, we define $\bK_S:=\bS^t \bK \bS$. Using these notation, the objective function of the quadratic GAN can be re-written as:
\begin{align}\label{eq:quadratic-gan-objective-subspace}
J(\bS,\bG_S,\bH_S)=&\Tr\left[ \left(\bI-\bH_S\bH_S^t\right) \bK_S \right]- \Tr\left[\left(\left(\bH_S\bH_S^t\right)^{\dagger}-\bI\right) \bG_S \bG_S^t \right]+\Tr\left[\bK-\bK_S\right].
\end{align}
Note that the first two terms of this objective is the same as \eqref{eq:quadratic-gan-objective} where all variables are projected to the column-space of $\bS$. Using the above argument, we propose the following {\it min-min-max} optimization:

\begin{align}\label{opt:min-min-max}
\min_{\bS}~ \min_{\bG_S} ~ \max_{\bH_S}~ &J(\bS,\bG_S,\bH_S)\\
&\bS^t\bS=\bI.\nonumber
\end{align}
The inner min-max optimization over $\bG_S$ and $\bH_S$ for a given $\bS$ is similar to the full-rank case analysis (Theorem \ref{thm:stability-1-1}). Given the global convergence of the $(1,1)$-alternating GD in the full-rank case, the outer optimization on $\bS$ can be re-written as

\begin{align}\label{opt:subspace-opt}
\max_{\bS}~ &\Tr\left[\bS^t \bK \bS\right]\\
&\bS^t\bS=\bI.\nonumber
\end{align}

Although this optimization is non-convex, it has been shown that its global optimizer, which recovers the leading eigenvectors of $\bK$, can be computed efficiently using GD \cite{ge2016efficient}.

An alternative approach to solve the quadratic GAN optimization \eqref{opt:quadratic-gan-empirical} is to solve the max part as a closed form and use GD to solve the min part. We analyze the convergence of this approach in Appendix Theorem \ref{thm:primal}.

\section{Discussion}\label{sec:beyond}

Our experiments on state-of-the-art GAN architectures suggest limitations of model-free designs even when data comes from a very basic Gaussian model. This motivates us to take a model-based approach to designing GANs. In this paper, we accomplish this goal in the spacial case of Gaussian models. Even though this is for a restrictive case, we have learnt a few lessons which will be useful as we broaden our approach. We obtained a general way to specify loss functions for GANs, by connecting the unsupervised GAN learning problem to the supervised learning problem. The quadratic loss function used for the Gaussian problem is a special case of this general connection. Moreover, we learnt that by properly constraining the class of generators and the class of discriminators in a {\em balanced} way, we can preserve good population solution while allowing fast generalization. Finally, we saw that using a model-based design, we could analyze the global stability of different computational approaches using gradient descent. These properties are hard to come by in model-free designs.

Our framework can potentially be used to design GANs for more complex distributions. For example, consider an {\em error-free} GAN architecture where there exists $\bG^*\in \cG$ such that $\PP_{\bG^*(X)}=\PP_Y$. The key question is how to design a {\em balanced} discriminator function for a given generator class $\cG$, i.e. if $\cG$ is the set of neural network functions with $l$ layers each with $m$ neurons, what should be the discriminator function set? We provide a non-parametric answer to this question in Appendix Section \ref{sec:SM-beyond-Gaussian}. A parametric characterization of the discriminator class $\cD$ for a given generator class $\cG$ is an interesting future direction for a model-based view to designing GANs.

\section{Acknowledgment}
We would like to thank Changho Suh, Fei Xia and Jiantao Jiao for helpful discussions.


\begin{thebibliography}{10}

\bibitem{goodfellow2014generative}
Ian Goodfellow, Jean Pouget-Abadie, Mehdi Mirza, Bing Xu, David Warde-Farley,
  Sherjil Ozair, Aaron Courville, and Yoshua Bengio.
\newblock Generative adversarial nets.
\newblock In {\em Advances in neural information processing systems}, pages
  2672--2680, 2014.

\bibitem{ledig2016photo}
Christian Ledig, Lucas Theis, Ferenc Husz{\'a}r, Jose Caballero, Andrew
  Cunningham, Alejandro Acosta, Andrew Aitken, Alykhan Tejani, Johannes Totz,
  Zehan Wang, et~al.
\newblock Photo-realistic single image super-resolution using a generative
  adversarial network.
\newblock {\em arXiv preprint arXiv:1609.04802}, 2016.

\bibitem{reed2016generative}
Scott Reed, Zeynep Akata, Xinchen Yan, Lajanugen Logeswaran, Bernt Schiele, and
  Honglak Lee.
\newblock Generative adversarial text to image synthesis.
\newblock {\em arXiv preprint arXiv:1605.05396}, 2016.

\bibitem{arjovsky2017wasserstein}
Martin Arjovsky, Soumith Chintala, and L{\'e}on Bottou.
\newblock Wasserstein gan.
\newblock {\em arXiv preprint arXiv:1701.07875}, 2017.

\bibitem{gulrajani2017improved}
Ishaan Gulrajani, Faruk Ahmed, Martin Arjovsky, Vincent Dumoulin, and Aaron
  Courville.
\newblock Improved training of wasserstein gans.
\newblock {\em arXiv preprint arXiv:1704.00028}, 2017.

\bibitem{sanjabi2018solving}
Maziar Sanjabi, Jimmy Ba, Meisam Razaviyayn, and Jason~D Lee.
\newblock Solving approximate wasserstein gans to stationarity.
\newblock {\em arXiv preprint arXiv:1802.08249}, 2018.

\bibitem{he2015delving}
Kaiming He, Xiangyu Zhang, Shaoqing Ren, and Jian Sun.
\newblock Delving deep into rectifiers: Surpassing human-level performance on
  imagenet classification.
\newblock In {\em Proceedings of the IEEE international conference on computer
  vision}, pages 1026--1034, 2015.

\bibitem{glorot2010understanding}
Xavier Glorot and Yoshua Bengio.
\newblock Understanding the difficulty of training deep feedforward neural
  networks.
\newblock In {\em Proceedings of the thirteenth international conference on
  artificial intelligence and statistics}, pages 249--256, 2010.

\bibitem{clevert2015fast}
Djork-Arn{\'e} Clevert, Thomas Unterthiner, and Sepp Hochreiter.
\newblock Fast and accurate deep network learning by exponential linear units
  (elus).
\newblock {\em arXiv preprint arXiv:1511.07289}, 2015.

\bibitem{nagarajan2017gradient}
Vaishnavh Nagarajan and J~Zico Kolter.
\newblock Gradient descent gan optimization is locally stable.
\newblock In {\em Advances in Neural Information Processing Systems}, pages
  5591--5600, 2017.

\bibitem{roth2017stabilizing}
Kevin Roth, Aurelien Lucchi, Sebastian Nowozin, and Thomas Hofmann.
\newblock Stabilizing training of generative adversarial networks through
  regularization.
\newblock In {\em Advances in Neural Information Processing Systems}, pages
  2015--2025, 2017.

\bibitem{mescheder2017numerics}
Lars Mescheder, Sebastian Nowozin, and Andreas Geiger.
\newblock The numerics of gans.
\newblock In {\em Advances in Neural Information Processing Systems}, pages
  1823--1833, 2017.

\bibitem{daskalakis2017training}
Constantinos Daskalakis, Andrew Ilyas, Vasilis Syrgkanis, and Haoyang Zeng.
\newblock Training gans with optimism.
\newblock {\em arXiv preprint arXiv:1711.00141}, 2017.

\bibitem{miyato2018spectral}
Takeru Miyato, Toshiki Kataoka, Masanori Koyama, and Yuichi Yoshida.
\newblock Spectral normalization for generative adversarial networks.
\newblock {\em arXiv preprint arXiv:1802.05957}, 2018.

\bibitem{petzka2017regularization}
Henning Petzka, Asja Fischer, and Denis Lukovnicov.
\newblock On the regularization of wasserstein gans.
\newblock {\em arXiv preprint arXiv:1709.08894}, 2017.

\bibitem{guo2017relaxed}
Xin Guo, Johnny Hong, Tianyi Lin, and Nan Yang.
\newblock Relaxed wasserstein with applications to gans.
\newblock {\em arXiv preprint arXiv:1705.07164}, 2017.

\bibitem{nowozin2016f}
Sebastian Nowozin, Botond Cseke, and Ryota Tomioka.
\newblock f-gan: Training generative neural samplers using variational
  divergence minimization.
\newblock In {\em Advances in Neural Information Processing Systems}, pages
  271--279, 2016.

\bibitem{dziugaite2015training}
Gintare~Karolina Dziugaite, Daniel~M Roy, and Zoubin Ghahramani.
\newblock Training generative neural networks via maximum mean discrepancy
  optimization.
\newblock {\em arXiv preprint arXiv:1505.03906}, 2015.

\bibitem{li2015generative}
Yujia Li, Kevin Swersky, and Rich Zemel.
\newblock Generative moment matching networks.
\newblock In {\em Proceedings of the 32nd International Conference on Machine
  Learning (ICML-15)}, pages 1718--1727, 2015.

\bibitem{mao2016multi}
Xudong Mao, Qing Li, Haoran Xie, Raymond~YK Lau, and Zhen Wang.
\newblock Multi-class generative adversarial networks with the l2 loss
  function.
\newblock {\em arXiv preprint arXiv:1611.04076}, 2016.

\bibitem{berthelot2017began}
David Berthelot, Tom Schumm, and Luke Metz.
\newblock Began: Boundary equilibrium generative adversarial networks.
\newblock {\em arXiv preprint arXiv:1703.10717}, 2017.

\bibitem{arora2017generalization}
Sanjeev Arora, Rong Ge, Yingyu Liang, Tengyu Ma, and Yi~Zhang.
\newblock Generalization and equilibrium in generative adversarial nets (gans).
\newblock {\em arXiv preprint arXiv:1703.00573}, 2017.

\bibitem{villani2008optimal}
C{\'e}dric Villani.
\newblock {\em Optimal transport: old and new}, volume 338.
\newblock Springer Science \& Business Media, 2008.

\bibitem{liu2017approximation}
Shuang Liu, Olivier Bousquet, and Kamalika Chaudhuri.
\newblock Approximation and convergence properties of generative adversarial
  learning.
\newblock {\em arXiv preprint arXiv:1705.08991}, 2017.

\bibitem{chernozhukov2017monge}
Victor Chernozhukov, Alfred Galichon, Marc Hallin, Marc Henry, et~al.
\newblock Monge--kantorovich depth, quantiles, ranks and signs.
\newblock {\em The Annals of Statistics}, 45(1):223--256, 2017.

\bibitem{rippl2016limit}
Thomas Rippl, Axel Munk, and Anja Sturm.
\newblock Limit laws of the empirical wasserstein distance: Gaussian
  distributions.
\newblock {\em Journal of Multivariate Analysis}, 151:90--109, 2016.

\bibitem{freund1996game}
Yoav Freund and Robert~E Schapire.
\newblock Game theory, on-line prediction and boosting.
\newblock In {\em Proceedings of the ninth annual conference on Computational
  learning theory}, pages 325--332. ACM, 1996.

\bibitem{ge2016efficient}
Rong Ge, Chi Jin, Praneeth Netrapalli, Aaron Sidford, et~al.
\newblock Efficient algorithms for large-scale generalized eigenvector
  computation and canonical correlation analysis.
\newblock In {\em International Conference on Machine Learning}, pages
  2741--2750, 2016.

\bibitem{cannon1987group}
James~W Cannon and William~P Thurston.
\newblock Group invariant peano curves.
\newblock 1987.

\bibitem{canas2012learning}
Guillermo Canas and Lorenzo Rosasco.
\newblock Learning probability measures with respect to optimal transport
  metrics.
\newblock In {\em Advances in Neural Information Processing Systems}, pages
  2492--2500, 2012.

\bibitem{weed2017sharp}
Jonathan Weed and Francis Bach.
\newblock Sharp asymptotic and finite-sample rates of convergence of empirical
  measures in wasserstein distance.
\newblock {\em arXiv preprint arXiv:1707.00087}, 2017.

\end{thebibliography}

\newpage


\begin{appendices}

\section{More Details On Numerical Experiments}\label{SMsec:exp}
For WGAN+WC and WGAN+GP, we use the implementation of \cite{gulrajani2017improved}. Neural net generators and discriminators have 3 hidden layers each with 64 neurons. We use the batch size of 200. WGAN-GP uses the Adam optimizer with a learning rate of $10^{-4}$, $\beta_1=0.5$ (i.e., the exponential decay rate for the first moment estimates) and $\beta_2=0.9$ (the exponential decay rate for the second moment estimates). WGAN-WC uses RMSProp optimizer with a learning rate of $5\times 10^{-5}$. In alternating gradient descent, five gradient steps for the discriminator updates and one gradient step for the generator updates are used. $\lambda$ parameter for WGAN+GP is set to be 0.1 (these parameters are set in the pipeline of \cite{gulrajani2017improved} for the Gaussian case).

Figure \ref{fig:nonlinearG-random} shows the performance of WGAN+GP and WGAN+WC in different values of $r$ with different initialization procedures when the generator and the discriminator functions are both neural networks. The setup is the same as the one of Figure \ref{fig:nonlinearG}. The only difference is that in experiments of Figure \ref{fig:nonlinearG-random}, we use a random covariance $\bK$ instead of the normalized identity covariance. To generate a random $\bK$, first we generate $\bW \bSigma \bW^t$ where $\bW(i,j)\sim\cN(0,1)$ and $\Sigma$ is a diagonal matrix where $\bSigma(i,i)\sim\text{uniform}(0,10)$. Then, we normalize $\bW \bSigma \bW^t$ to have unit Frobenius norm.

In the linear generator case, we restrict the generator to linear functions since both the observed data and the randomness have Gaussian distributions. In the case of $r=d$, a maximum-likelihood generative model is $\cN(\tilde{\bmu},\tilde{\bK})$ where $\tilde{\bmu}$ and $\tilde{\bK}$ are the sample mean and the sample covariance, respectively. In the case where $r<d$, a maximum-likelihood generative model is $\cN(\tilde{\bmu},\tilde{\bK}_r)$ where $\tilde{\bK}_r$ is the $r$-PCA of $\tilde{\bK}$, i.e., assuming that $\tilde{\bK}$ has the eigen decomposition of $\tilde{\bK}=\sum_{i=1}^{d}\lambda_i \bv_i \bv_i^t$ where $\lambda_1\geq \lambda_2\geq ...$, then $\hat{\bK}_r=\sum_{i=1}^{r}\lambda_i \bv_i \bv_i^t$. These generative models act as our {\it baselines} to evaluate GAN's performance in this case.

\section{Computation of the Quadratic GAN}\label{sec:implementaion}
We consider two approaches to compute a solution for the quadratic GAN. In the $(1,1)$-alternating gradient descent approach, if $\bH$ is full rank, we have
\begin{align}
\nabla_{\bG} J(\bG,\bH)&= 2 \left(\bI-\bA^{-1}\right)\bG,\\
\nabla_{\bH} J(\bG,\bH)&= -2\tbK \bH+2 \bA^{-1} \bH \bH^{-t},\nonumber
\end{align}
where $\bA=\bH\bH^t$. We use these explicit formula to compute a solution for the quadratic GAN in Figure \ref{fig:quadratic_gan}-b) experiments.

Another approach to compute a solution for the quadratic GAN is to solve the inner maximization in a closed-form and then use GD to solve the minimization part. We refer to this approach as the {\it primal} approach. We explain this approach below:

\begin{figure}[t]
\centering
\includegraphics[width=0.8\linewidth]{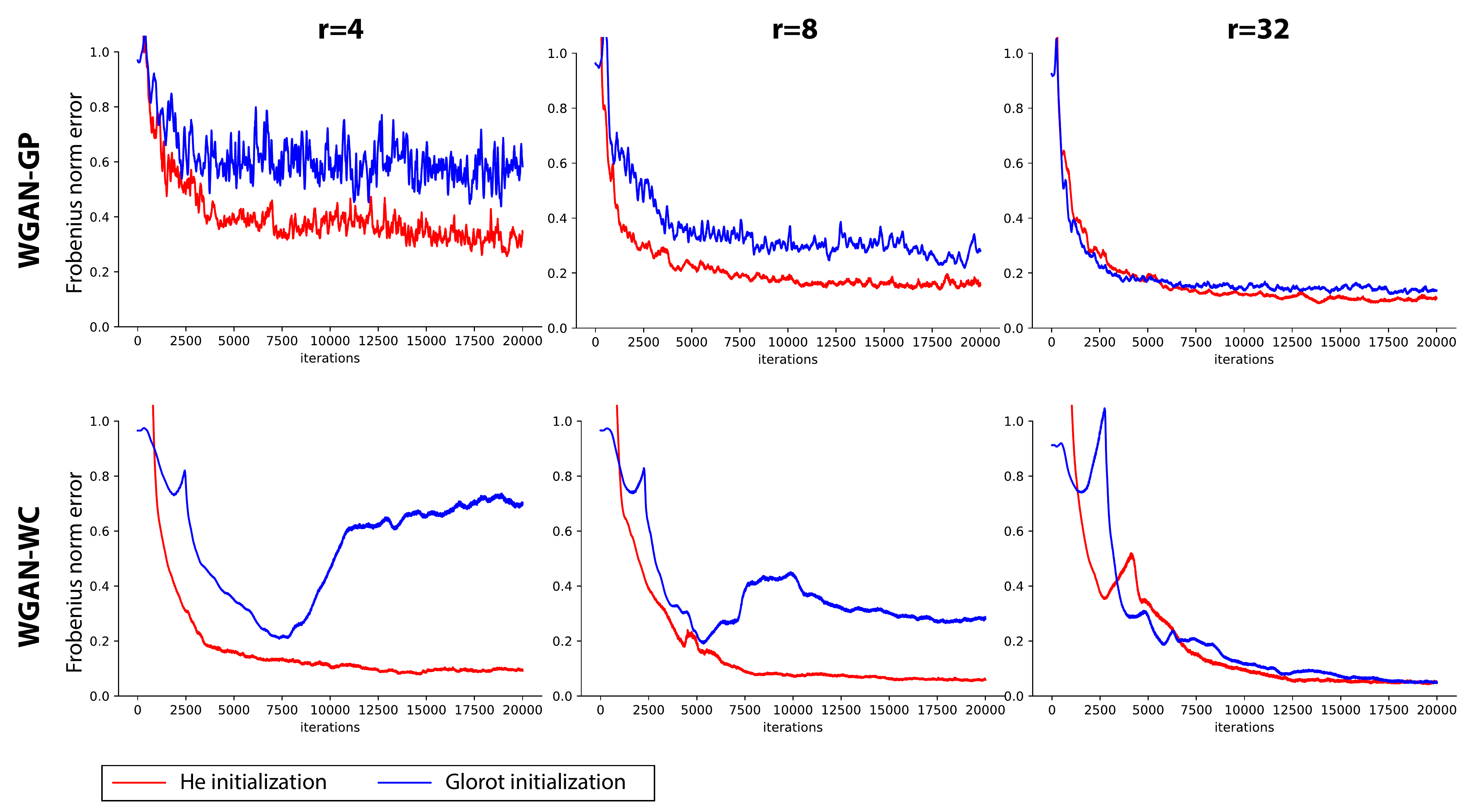}
\caption{An illustration of the performance of WGAN+GP and WGAN+WC in different values of $r$ with different initialization procedures when the generator and the discriminator functions are both neural networks. In these experiments, $\bK$ is a random covariance matrix.}
\label{fig:nonlinearG-random}
\end{figure}

\begin{figure}[t]
\centering
\includegraphics[width=0.7\linewidth]{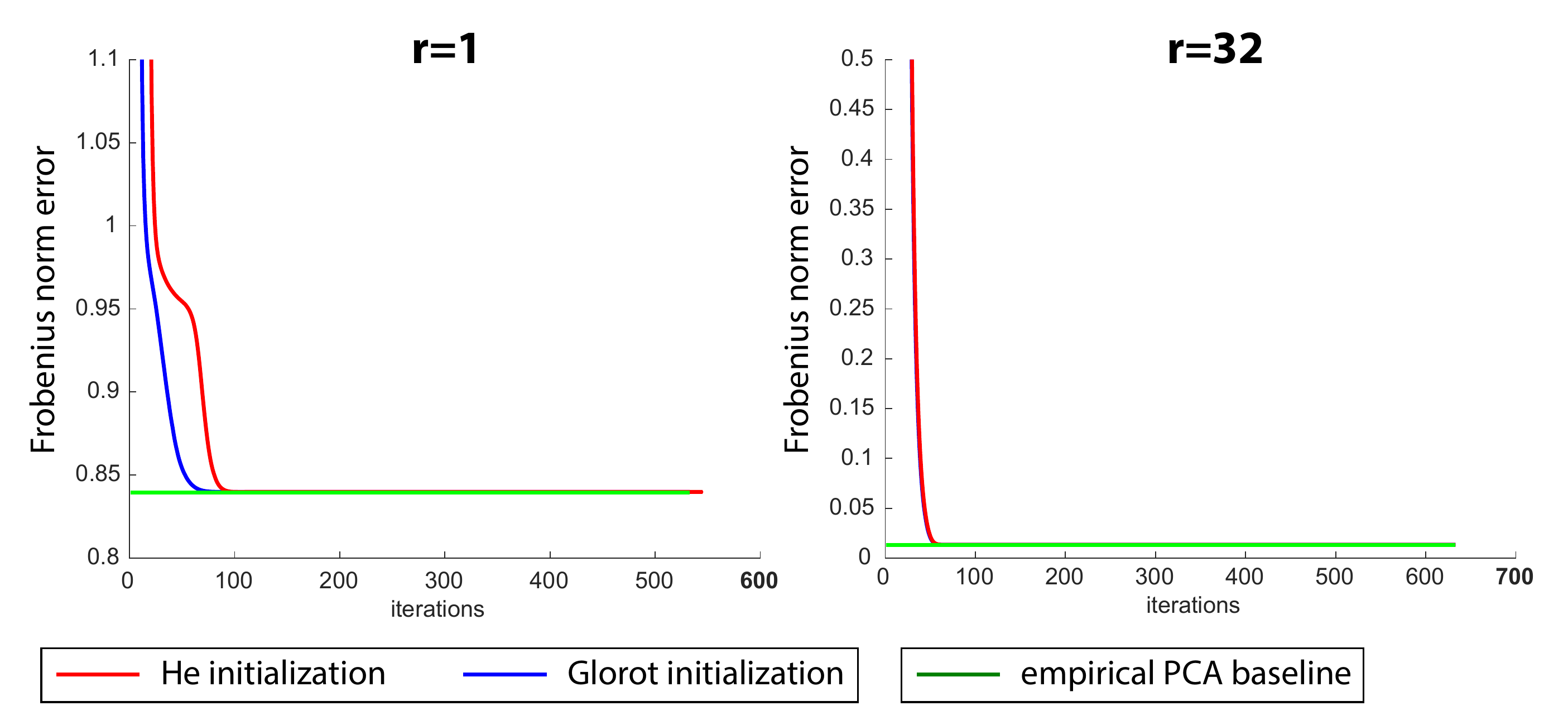}
\caption{An illustration of the performance of Quadratic GAN in different values of $r$ with different initialization procedures using the primal approach explained in Appendix Section \ref{sec:implementaion}.}
\label{fig:1-inf}
\end{figure}

Define
\begin{align}\label{eq:h-define}
h(\bG)=\max_{\bH}~ J(\bG,\bH).
\end{align}
Let $\bH^*$ be the optimal $\bH$ for a given $\bG$. From Lemma \ref{lem:det-coupling}, we have
\begin{align}\label{eq:H-opt}
\left(\bH^*\right)^2= \tbK^{-\frac{1}{2}} \left(\tbK^{\frac{1}{2}}\bG \bG^t \tbK^{\frac{1}{2}}  \right)^{\frac{1}{2}}\tbK^{-\frac{1}{2}}.
\end{align}
Using \eqref{eq:H-opt} in \eqref{eq:h-define}, we have
\begin{align}\label{eq:h-function}
h(\bG)=\Tr\left[\tbK+\bG\bG^t -2\left(\tbK^{\frac{1}{2}}\bG \bG^t \tbK^{\frac{1}{2}}  \right)^{\frac{1}{2}} \right].
\end{align}
In the primal approach, we solve the following optimization: 
\begin{align}\label{opt:primal}
\min_{\bG}~~ h(\bG)
\end{align}

\begin{theorem}\label{thm:primal}
In the quadratic GAN optimization \eqref{opt:quadratic-gan-empirical},
\begin{itemize}
  \item [(a)] For $r=1$, the primal approach of optimization \eqref{opt:primal} is globally stable.
  \item [(b)] For $r=d$, the primal approach of optimization \eqref{opt:primal} if the initial $\bG$ is symmetric.
\end{itemize}
\end{theorem}
\begin{proof}
See Appendix Section \ref{proof:primal}.
\end{proof}

If $r=d$ and for a PSD initialization of $\bG$, we have
\begin{align}\label{eq:der-h-full}
\nabla_{\bG}~ h(\bG)=2\bG-2\tbK^{\frac{1}{2}}.
\end{align}
Moreover, in the proof of Theorem \ref{thm:primal}-part (b), we show that with such an initialization, $\bG$ remains PSD along the gradient descent trajectory. Thus, the above gradient formula holds.

In the case where $r<d$, the gradient calculation is a bit more challenging. Here we focus on the case where $r=1$. In this case, $\bG$ is a vector, thus we represent it by $\bg$. We can write $\bg=a \bv$ where $a=\|\bg\|$ and $\|\bv\|=1$.

Let $\tbG:=(\bg\bg^t)^{1/2}$, i.e., $\tbG$ is the PSD square root of the matrix $\bg\bg^t$. We have
\begin{align}
 \tG=(\bg\bg^t)^{\frac{1}{2}}=a (\bv \bv^t)^{\frac{1}{2}}=a \bv \bv^t=\frac{\bg \bg^t}{\|\bg\|}.
\end{align}
Thus, we have
\begin{align}\label{eq:h-r=1}
h(\bg)=\Tr\left[\tbK+\bg\bg^t-2\tbK^{\frac{1}{2}} \frac{\bg\bg^t}{\|\bg\|} \right].
\end{align}
Thus, we have
\begin{align}\label{eq:der-r=1}
\nabla_{\bg}~ h(\bg)=2\bg-2\tbK^{\frac{1}{2}}\frac{\bg}{\|\bg\|}.
\end{align}

Figure \ref{fig:1-inf} shows our empirical result for the primal approach when $\bK$ is a random matrix (the same matrix as the one used in Figure \ref{fig:nonlinearG-random}). As illustrated in this figure, quadratic GAN recovers the maximum-likelihood/PCA solution. In these implementations, we use MATLAB's ordinary differential equation solver (ode45).


\section{Generalization Issues of GANs}\label{sec:generalization}
Consider the W2GAN optimization \eqref{opt:w2gan}. In practice, one solves the GAN optimization over the empirical distribution of the data $\QQ_Y^n$, not the population distribution $\PP_Y$. Thus, it is important to analyze how close optimal empirical and population GAN solutions are in a given sample size $n$. This notion is captured in the generalization error of the GAN optimization, defined as follows:
\begin{definition}\label{def:gen-unsup}
Let $n$ be the number of observed samples from $Y$. Let $\hbG(.)$ and $\bG^*(.)$ be the optimal generators for empirical and population W2GANs respectively. Then,
\begin{align}\label{eq:gen-bound}
d_{\cG}(\PP_Y,\QQ_Y^n):=W_2^2(\PP_Y,\PP_{\hbG(X)})-W_2^2(\PP_Y,\PP_{\bG^*(X)}),
\end{align}
is a random variable representing the excess error of $\hbG$ over $g^*$, evaluated on the true distribution.
\end{definition}
$d_{\cG}(\PP_Y,\QQ_Y^n)$ can be viewed as a distance between $\PP_Y$ and $\QQ_Y^n$ which depends on $\cG$. To have a proper generalization property, one needs to have $d_{\cG}(\PP_Y,\QQ_Y^n)\to 0$ quickly as $n\to \infty$. First, we characterize this rate for an unconstrained $\cG$. For an unconstrained $\cG$, the second term of \eqref{eq:gen-bound} is zero (this can be seen using a space filling generator function \cite{cannon1987group}). Moreover, $\PP_{\hbG(X)}$ can be arbitrarily close to $\QQ_Y^{n}$. Thus, we have
\begin{align}
d_{\cG}(\PP_Y,\QQ_Y^n)=W_2^2(\PP_Y,\QQ_Y^n),
\end{align}
which goes to zero with high probability with the rate of $\cO(n^{-2/d})$.

The approach described for the unconstrained $\cG$ corresponds to the {\it memorization} of the empirical distribution $\QQ_Y^n$ using the trained model. Note that one can write
\begin{align*}
n^{-\frac{2}{d}}=2^{-\frac{2\log(n)}{d}}.
\end{align*}
Thus, to have a small $W_2^2(\PP_Y,\QQ_Y^n)$, the number of samples $n$ should be exponentially large in $d$ \cite{canas2012learning}. It is possible that for some distributions $\PP_Y$, the convergence rate of $W_2^2(\PP_Y,\QQ_Y^n)$ is much faster than $n^{-2/d}$. For example, \cite{weed2017sharp} shows that if $\PP_Y$ is clusterable (i.e., $Y$ lies in a fixed number of separate balls with fixed radii), then the convergence of $W_2^2(\PP_Y,\QQ_Y^n)$ is fast. However, even in that case, one optimal strategy would be to memorize observed samples,  which does not capture what GANs do.

In supervised learning, constraining the predictor to be from a small family improves generalization. A natural question is whether constraining the family of generator functions $\cG$ can improve the generalization of GANs. In the Gaussian setting, we are constraining the generators to be linear. To simplify calculations, we assume that $Y\sim \cN(\mathbf{0},\bI_d)$ and $d=r$. Under these assumptions, the W2GAN optimization can be re-written as
\begin{align}\label{opt:fit-gaussian}
\min_{\bmu,\bK}~~ W_2^2 (\QQ_Y^n,\cN\left(\bmu,\bK\right)),
\end{align}
where $\bK$ is the covariance matrix. The optimal population solution of this optimization is $\bmu^*_{pop}=\mathbf{0}$ and $\bK^*_{pop}=\bI$, which provides a zero Wasserstein loss with respect to the true distribution.

In the following theorem, we characterize the convergence rate of the W2GAN optimization for linear generators with  single-parameters. In this case, $\bK=s^2 \bI$ ($\bK$ is a diagonal matrix whose diagonal elements are equal to $s^2$). Note that if $s=1$, the trained model matches the population distribution.

\begin{theorem}\label{thm:convergance}
Let $\bmu^*_{n}$ and $\bK^*_{n}=(s^*)^2 \bI$ be optimal solutions for optimization \eqref{opt:fit-gaussian} where $\bK$ is restricted to $s^2 \bI$ (i.e. the generator is a single-parameter linear function). Then, $s^*\to 1$ with the rate of $n^{-2/d}$.
\end{theorem}

Now, consider a ball around the distribution $\QQ_Y^n$ where $\PP_Y$ lies on its surface. Note that the radius of this ball is a random variable that is concentrated around $n^{-2/d}$ \cite{villani2008optimal}. This radius is large and goes to zero exponentially slow in $d$. If there is another Gaussian distribution inside this ball, the learner would select that distribution in the optimization rather than $\PP_Y$. The Gaussian distribution computed in Theorem \ref{thm:convergance} satisfies this condition. Thus, in this case, one needs $n$ to be exponentially large in $d$ to have the error go to zero. To enhance the convergence rate of GANs, in practice, discriminators are constrained. We discuss how discriminators should be constrained properly in the Gaussian benchmark in Section \ref{sec:qgan}.



\begin{figure}[t]
\centering
\includegraphics[width=0.4\linewidth]{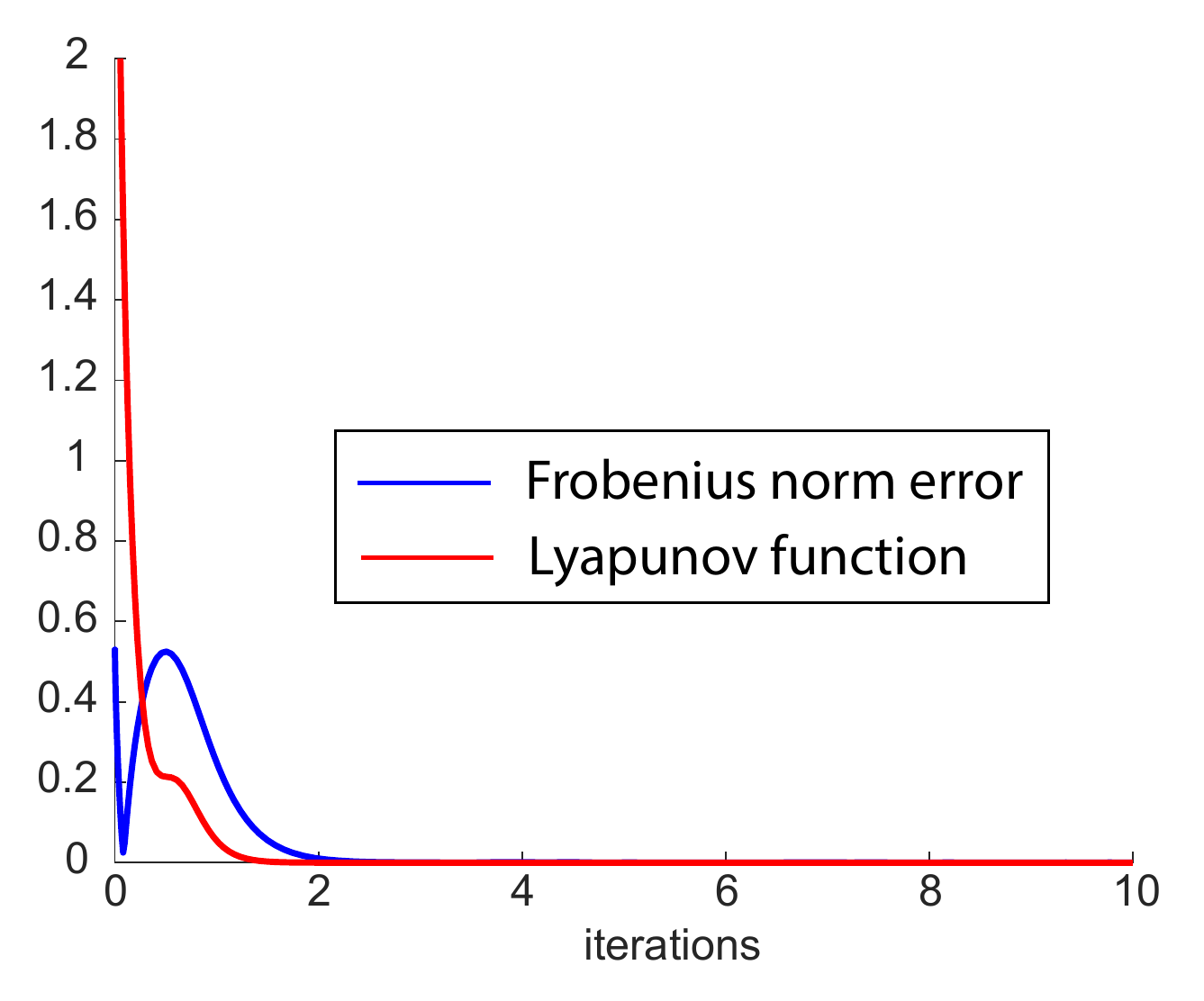}
\caption{Values of the Lyapunov function \eqref{eq:lyp} and $\|\bG\bG^t-\tbK\|_F$ over a trajectory of $(1,1)$-alternating gradient descent. The former is monotonically decreasing while the later is not. This fact is used to prove the global stability result of Theorem \ref{thm:stability-1-1}.}
\label{fig:lyp}
\end{figure}

\section{Beyond Gaussians}\label{sec:SM-beyond-Gaussian}
Consider the case where there exists an {\em error-free} GAN architecture; i.e., there exists $\bG^*\in \cG$ such that $\PP_{\bG^*(X)}=\PP_Y$. The key question is how to design a {\em balanced} discriminator function for a given generator class $\cG$.

\begin{theorem}\label{thm:non-guassian}
Let $\cG$ be the set of generator functions in the W2GAN optimization \eqref{opt:w2gan}. Let $\nabla \cD$ be the set of gradient functions of members of $\cD$. Let
\begin{align}\label{eq:disc-const-general}
\nabla \cD=\left\{F_{\bG_1}^{-1}\circ F_{\bG_2}: \forall \bG_1,\bG_2 \in \cG \right\}.
\end{align}
Then, constraining the discriminator to the intersection of convex function with $\cD$ does not change the optimal solution of the original W2GAN optimization.
\end{theorem}

\section{Proofs}\label{sec:proofs}

\subsection{Proof of Theorem \ref{thm:PCA}}
First we prove the following lemmas:

\begin{lemma}\label{lem:proj}
Let $\cS$ be an $r$ dimensional subspace in $\mathbb{R}^d$. Let $\hY$ be a random variable whose support lies in $\cS$. Then, $\hY^*$, the optimal solution of the optimization
\begin{align}\label{opt:proj}
\inf_{\PP_{\hY}}~~ W_2^2(\PP_Y,\PP_{\hY}),
\end{align}
is the projection of $Y$ to $\cS$.
\end{lemma}

\begin{proof}
Let $Y=Y_{\cS'}+Y_{\cS}$ where $Y_{\cS}$ represents the projection of $Y$ onto the subspace $\cS$. Since
\begin{align}
\EE\left[\|Y-\hY \|^2\right]=\EE\left[\|Y_{\cS'}\|^2\right]+\EE\left[\|Y_{\cS}-\hY \|^2\right]
\end{align}
choosing $\hY=Y_{\cS}$ achieves the minimum of optimization \eqref{opt:proj}.
\end{proof}

Let $\cS$ be a fixed subspace of rank $r$ where $\hY$ lies on. According to Lemma \ref{lem:proj}, if $\hY$ is unconstrained, the optimal $\hY^*$ is the projection of $Y$ onto $\cS$ (i.e., $\hY^*=Y_{\cS}$). Moreover, since $Y$ is Gaussian, $\hY$ is also Gaussian. Therefore, there exists a linear $\bG(.)$ such that $\PP_{\hY^*}=\PP_{\bG(X)}$ where $X\sim \cN(\mathbf{0},\mathbf{I})$. Thus, the problem simplifies to choosing a subspace where $\EE\left[\|Y_{\cS}\|^2\right]$ is maximized, which is the same as the PCA optimization.

\subsection{Proof of Theorem \ref{thm:eq-quality}}
Let $\tY$ be a random variable whose distribution matches the empirical distribution $\QQ_Y^n$. Using \eqref{opt:dual-w2}, we can write:
\begin{align}\label{opt:disc-coupling}
W_2^2(\PP_{\tY},\PP_{\bG(X)})=&\max_{\psi(.) \text{:convex}} ~\EE[\|\tY\|^2]+\EE[\|\bG(X)\|^2] -2\EE\left[\psi(\tY)\right]-2\EE\left[\psi^*(\bG(X))\right].
\end{align}

By constraining the discriminator to convex quadratic functions, we obtain \eqref{opt:quadratic-GAN} (as we show in Section \ref{sec:qgan}). We also have the following lower bound:
\begin{align}\label{opt:constrained-quadratic-gan-emp}
&W_2^2(\PP_{\tY},\PP_{\bG(X)})>\\
&\max_{\psi(\by)=\by^t \bA \by/2, \bA \succeq 0} \EE[\|\tY\|^2]+\EE[\|\bG(X)\|^2] -2\EE\left[\psi(\tY)\right]-2\EE\left[\psi^*(\bG(X))\right]\nonumber\\
&=\sup_{\bA \succeq 0}~~ \Tr(\tilde{\bK})+\Tr(\bK_{\bG(X)})-\Tr(\bA \tilde{\bK})-\Tr(\bA^{\dagger} \bK_{\bG(X)})\nonumber\\
&=W_2^2(\PP_{Z},\PP_{\bG(X)}),\nonumber
\end{align}
where $\tilde{\bK}=\EE[\tY \tY^t]$ (the empirical covariance matrix) and $Z\sim \cN(\mathbf{0},\tilde{\bK})$. Therefore, the quadratic GAN solves the following optimization:
\begin{align}\label{opt:GAN-emp-const}
\min_{\bG(.)\in\cG}~ W_2^2 (\PP_Z,\PP_{\bG(X)}).
\end{align}
Using Theorem \ref{thm:PCA}, the optimal $\bG(X)$ to this problem is the empirical PCA solution, i.e. keeping the top $r$ principal components of the {\em empirical covariance matrix}. This completes the proof.

\subsection{Proof of Theorem \ref{thm:stability-1-1}}
Since $\bH$ is full-rank, we ignore the constraint $\range(\bG)= \range(\bH)$. We have
\begin{align}\label{eq:der-J}
\nabla_{\bG} J(\bG,\bH)&= 2 \left(\bI-\bA^{-1}\right)\bG\\
\nabla_{\bH} J(\bG,\bH)&= -2 \bH+2 \bA^{-1} \bH \bH^{-t},\nonumber
\end{align}
where $J(\bG,\bH)$ is the objective function of the quadratic GAN optimization \eqref{eq:quadratic-gan-objective}. Solving $\nabla_{\bG} J(\bG,\bH)=0$ and $\nabla_{\bH} J(\bG,\bH)=0$ lead to a unique solution of $\bU^*=\tbK=\bI$ and $\bA^*=\bI$ (note that $\bG^*$ is not unique and can be any orthogonal matrix).

To show that $(1,1)$-alternating gradient descent is globally convergent to this solution, we show that function $V(.,.)$ of \eqref{eq:lyp} is a Lyapunov function. First, note that this function is non-negative because it is the sum of two Von Neumann divergence functions. Also, at the solution, $V=0$.

Moreover, gradients of the function $V$ can be written as
\begin{align}\label{eq:der-V}
\nabla_{\bG} V(\bG,\bH)&= 2 \left(\bG-\bG^{-t}\right),\\
\nabla_{\bH} V(\bG,\bH)&= 2 \left(\bH-\bH^{-t}\right).\nonumber
\end{align}

Thus, we have
\begin{align}
&-\left<\nabla_{\bG} J, \nabla_{\bG} V\right>+\left<\nabla_{\bH} J, \nabla_{\bH} V\right>\\
&=\frac{1}{4}\Tr\Big[-\left(\bI-\bA^{-1}\right)\bG \left(\bG^t-\bG^{-1}\right)+\left(-\bH+\bA^{-1}\bU \bH^{-t}\right)\left(\bH^{t}-\bH^{-1}\right) \Big]\nonumber\\
&=\frac{1}{4}\Tr\Big[-\left(\bI-\bA^{-1}\right)\left(\bU-\bI\right)+\left(-\bA+\bA^{-1}\bU\right)\left(\bI-\bA^{-1}\right)\Big]\nonumber\\ 
&=-\frac{1}{4}\Tr\left[(\bI-\bA^{-1})^2 (\bU+\bA)\right]\leq 0,\nonumber
\end{align}
where the last step follows from the fact that for two PSD matrices $\bX$ and $\bY$, $\Tr[\bX \bY]\geq 0$. Therefore, $V(\bG,\bH)$ is a Lyapunov function. This completes the proof.

\subsection{Proof of Theorem \ref{thm:primal}}\label{proof:primal}
First, we prove part (a).  Recall the definition of $h(\bG)$ \eqref{eq:h-function}. Let $(\lambda_i,\bu_i)$ be the $i$-th eigenvalue/eigenvector pair of $\tbK$. For simplicity, we assume $\tbK$ does not have eigenvalue multiplicity.


Note that to have $\nabla_{\bg} h(\bg)=0$ in \eqref{eq:der-r=1}, we need to have $\bg=\sqrt{\lambda_j}\bu_j$ for $1\leq j\leq d$. Below we show that $\bg^*=\sqrt{\lambda_1}\bu_1$ is the only stable solution among these solutions. To do this, we compute the Hessian matrix of $h(\bg)$ as follows:
\begin{align}
\nabla_{\bg}^2 \bh(\bg)=2\bI-2\tbK^{\frac{1}{2}} \left(\frac{\bI}{\|\bg\|}-\frac{\bg\bg^t}{\|\bg^3\|} \right).
\end{align}
We evaluate the Hessian at $\bg=\sqrt{\lambda_j}\bu_j$ for $1\leq j\leq d$:

\begin{align}
\nabla_{\bg}^2 \bh(\bg=\sqrt{\lambda_j}\bu_j)= 2\bu_j \bu_j^t+2\sum_{i\neq j} \left(1-\sqrt{\frac{\lambda_i}{\lambda_j}}\right).
\end{align}
The Hessian matrix is PSD when $j=1$ and has a negative eigenvalue when $j>1$. This completes the proof of part (a).

Next, we prove part (b). First, we consider the case where $\bG$ is PSD and symmetric. In this case, the third term in \eqref{eq:h-function} is equal to $\tbK^{1/2}\bG$ which is linear in $\bG$. Thus, the function $h(\bG)$ is convex. We also have
\begin{align}
\nabla_{\bG} h(\bG)=2\bG-2\tbK^{\frac{1}{2}}.
\end{align}
Thus, if the initialization of $\bG$ is PSD, $\bG$ will remain PSD along the trajectory.

Finally, consider the case where $\bG$ is not PSD and has some negative eigenvalues. Since $h(.)$ depends on $\bG$ though $\bG\bG^t$, the landscape of regions where $\bG$ has negative eigenvalues is the same as the region where $\bG$ is PSD. This completes the proof of part (b).

\subsection{Proof of Theorem \ref{thm:convergance}}\label{proof:thm:convergence}
Consider the case where $\bK=s^2 \bI$ ($\bK$ is a diagonal matrix whose diagonal elements are equal to $s^2$). In this case, we solve the following optimization:

\begin{align}\label{opt:fit-gaussian-diag}
\min_{\bmu,s}~~ W_2^2 (\QQ_Y^n,\cN\left(\bmu,s^2\bI\right)).
\end{align}


Let $\by_1$,...,$\by_n$ be $n$ i.i.d. samples of $\PP_Y$. Let $\hat{\bmu}$ be the sample mean. Since $\PP_Y$ is absolutely continuous, the optimal $W_2$ coupling between $\QQ_Y^n$ and $\PP_Y$ is deterministic \cite{villani2008optimal}. Thus, every point $\by_i$ is coupled with an optimal transport vornoi region with the centroid $\bc_{y_i}^{(\bmu,\bK)}$. Therefore, we have
\begin{align}\label{eq:finite-n-decompos}
&W_2^2(\cN\left(\bmu,\bK\right),\QQ_Y^n)\\
&=\|\bmu\|^2+\Tr(\bK)+\frac{1}{N}\sum_{i=1}^{N}\|\by_i\|^2-\frac{2}{N} \sum_{i=1}^{N}\by_i^t \bc_{y_i}^{(\bmu,\bK)}\nonumber\\
&=\|\bmu\|^2+\Tr(\bK)+\frac{1}{N}\sum_{i=1}^{N}\|\by_i\|^2-\frac{2}{N} \sum_{i=1}^{N}\by_i^t \left(\bK^{1/2}\bc_{y_i}^{(\mathbf{0},\bI)}+\bmu\right)\nonumber\\
&=\|\bmu\|^2-2\bmu \hat{\bmu}+ \Tr(\bK)+\frac{1}{N}\sum_{i=1}^{N}\|\by_i\|^2-2 \Tr (\bK^{1/2} \bA)\nonumber
\end{align}
where
\begin{align}\label{eq:A}
\bA:=\frac{1}{N}\sum_{i=1}^{N} \bc_{y_i}^{(\mathbf{0},\bI)} \by_i^t.
\end{align}
The first step in \eqref{eq:finite-n-decompos} follows from the definition of $W_2$, the second step follows from the optimal coupling between $\cN(\bmu,s^2 \bI)$ and $\cN(\mathbf{0},\bI)$, and the third step follows from the matrix trace equalities.

Therefore,
\begin{align}
\bigtriangledown_{\bmu} W_2^2(\cN\left(\bmu,\bSigma\right),\QQ_Y^n)=2\bmu-2\hat{\bmu},
\end{align}
which leads to $\bmu^*_{n}=\hat{\bmu}$. Moreover, each component of the sample mean is distributed according to $\cN(\mathbf{0},1/n)$. Thus, $\|\bmu^*_N\|^2\sim \chi_d^2/n$ which goes to zero with the rate of $\tO(d/n)$.

Moreover, we have
\begin{align}
\bigtriangledown_{s} W_2^2(\cN\left(\bmu,\bK\right),\QQ_Y^n)=2sd-2\Tr(\bA).
\end{align}
Thus, $s^*=\Tr(\bA)/d$.

Furthermore, we have
\begin{align}
 W_2^2(\PP_Y,\QQ_Y^n)=d+\frac{1}{N}\sum_{i=1}^{N}\|\by_i\|^2-2\Tr(\bA),
\end{align}
which goes to zero with the rate of $\cO(n^{-2/d})$ \cite{canas2012learning}. Since $\frac{1}{N}\sum_{i=1}^{N}\|\by_i\|^2$ goes to $d$ with the rate of $\cO(\sqrt{d/n})$ (because it has a $\chi$-squared distribution), $\Tr(\bA)$ goes to $d$ with a rate of $\cO(n^{-2/d})$. Thus, $s^*$ goes to one with a rate of $\cO(n^{-2/d})$.

\subsection{Proof of Theorem \ref{thm:non-guassian}}
Let $Y=\bG^*(X)$. Thus, $F_{\bG^*}(Y)$ has a uniform distribution. From the inverse transform sampling, $F_{\bG}\circ F_{\bG^*}^{-1}(X)$ has a distribution same as $\PP_{\bG(X)}$. Using Lemma \ref{lem:det-coupling} completes the proof.

\end{appendices}

\end{document}